\documentclass[letterpaper, 10pt, conference]{ieeeconf}

\IEEEoverridecommandlockouts   
\usepackage[utf8]{inputenc}
\usepackage[T1]{fontenc}

\usepackage{amsmath,amssymb}
\usepackage{graphicx}
\usepackage{booktabs}
\usepackage{array}
\usepackage{multirow}
\usepackage{caption}
\usepackage{subcaption}
\usepackage{xcolor}
\usepackage{cite}
\usepackage{bm}
\usepackage{amsthm}
\theoremstyle{plain}
\newtheorem{theorem}{Theorem}
\newtheorem{lemma}{Lemma}
\newtheorem{proposition}{Proposition}
\newtheorem{corollary}{Corollary}
\theoremstyle{definition}
\newtheorem{definition}{Definition}
\newtheorem{assumption}{Assumption}
\newtheorem{problem}{Problem}
\theoremstyle{remark}
\newtheorem{remark}{Remark}
\makeatletter
\let\NAT@parse\undefined
\makeatother
\usepackage[colorlinks=true,allcolors=blue]{hyperref}
\usepackage{algorithm}
\usepackage{algpseudocode}

\newcommand{\R}{\mathcal{R}}                 
\newcommand{\Traj}{\mathcal{T}}              
\newcommand{\Ifam}{\mathcal{I}}              
\newcommand{\Iind}{\mathcal{I}_{\mathrm{ind}}} 
\newcommand{\Pcfg}{P}                        
\newcommand{\kf}{\kappa_f}                   
\newcommand{\nuf}{\nu_f(\Ifam)}              
\newcommand{\rc}{r_{c}}                      
\newcommand{\rs}{r_{s}}                      
\newcommand{\rmove}{r_{\mathrm{move}}}       
\newcommand{\Ares}{A_{\mathrm{RES}}}         
\newcommand{\Aone}{A_{1}}                     
\newcommand{\Atwo}{A_{2}}                     
\newcommand{\fopt}{f^{\star}}                
\newcommand{\foptI}{f^{\star}_{\Ifam}}       
\newcommand{\AoptI}{A^{\star}_{\Ifam,2}}     
\newcommand{\Bstar}{B^{\star}}               
\newcommand{\gunc}{g^{\mathrm{unc}}}         
\newcommand{\gcon}{g^{\mathrm{con}}}         
\newcommand{\hNb}{h(N,\beta)}                
\newcommand{\marg}[2]{f(#1 \mid #2)}         
\newcommand{\setdiff}{\setminus}
\newcommand{\eqdef}{:=}

\title{\LARGE \bf
Failure-Resilient Multi-Robot Coordination with $k$-Connectivity Constrained Submodular Maximization
}

\author{Yoonjae Hwang$^{1}$, Hojune Kim$^{1}$,
Ellen Novoseller$^{2}$, Guangyao Shi$^{1}$,and Gaurav S. Sukhatme$^{1}$
\thanks{$^{1}$Y.~Hwang, G.~Shi, H.~Kim, and G.~S. Sukhatme are
         with the School of Computer Science, University of Southern California,
         Los Angeles, CA, USA
         {\tt\small \{yoonjae, shig, hojuneki, gaurav\}@usc.edu}}%
\thanks{$^{2}$E.~Novoseller is with the U.S. Army DEVCOM Army Research Laboratory,
         {\tt\small ellen.r.novoseller.civ@army.mil}}%
}

\begin{document}

\maketitle
\thispagestyle{empty}
\pagestyle{empty}

\begin{abstract}
Multi-robot coordination is essential in applications such as coverage, target tracking, and information gathering, where robots must jointly select actions to maximize team performance. Submodular maximization provides an effective framework for such coordination, and resilient variants further account for robot failures or adversarial removals. However, existing methods typically overlook the impact of robot loss on the communication network, which may disconnect the surviving team and prevent continued coordination. To address this issue, we introduce Resilient \(k\)-Connected Submodular Maximization, which jointly considers task performance, resilience, and communication connectivity. The resulting feasible family is not a matroid, preventing a direct application of classical curvature-greedy guarantees. We overcome this challenge through a Menger-based connectivity reduction, a connectivity-aware greedy algorithm with a feasibility-preserving fallback, and a new curvature-based analysis that captures the cost of enforcing connectivity through a single factor \(\rho\). Our method achieves a worst-case post-failure objective within a factor
$
[\max(1-\nu_f(\mathcal{I}),h(N,\beta))]^2 \rho \frac{1}{\kappa_f}(1-e^{-\kappa_f})
$
of the optimal connectivity-respecting resilient assignment, while recovering the classical unconstrained guarantee in the appropriate limits. Experiments on multi-robot target tracking validate the proposed method and characterize the practical impact of connectivity constraints.

\end{abstract}

\section{Introduction}
\label{sec:intro}


Multi-robot systems provide scalable solutions to applications such as environmental monitoring, target tracking, search and rescue, and collaborative exploration~\cite{marques2015critical, robin2016, queralta2020, burgard2005}. In many of these settings, the robots must coordinate their sensing or motion decisions to maximize a common team objective~\cite{krausegolovin2014}. Submodular set functions provide a natural mathematical framework for such coordination by capturing the principle of diminishing returns: the marginal benefit of adding an additional robot decreases as more robots are already selected. Many commonly used objectives, including mutual information~\cite{krause2008robust}, area coverage~\cite{shi2021}, number of observed targets~\cite{zhou2019}, and sensor-measurement entropy~\cite{dames2017} are monotone submodular~\cite{nemhauser1978}. This structure enables computationally efficient greedy algorithms with well-established constant-factor performance guarantees, as reviewed in Section~\ref{sec:rel}.

\begin{figure}[t]
  \centering
  \includegraphics[width=\columnwidth]{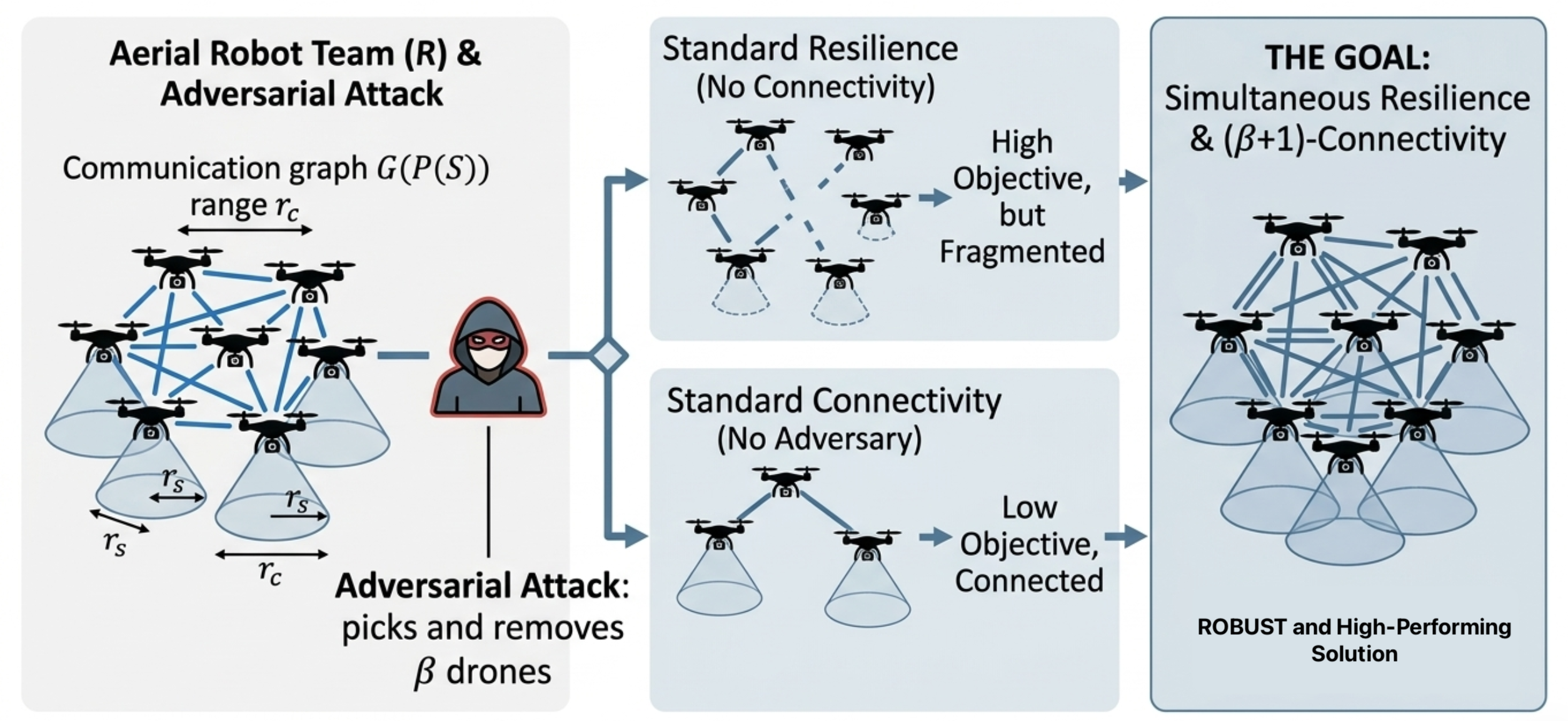}
  \caption{The core structural conflict in multi-robot deployments. \textbf{(Left)} Aerial teams face adversarial attacks or failures removing up to $\beta$ nodes. \textbf{(Middle)} Existing frameworks treat requirements in isolation, in which standard resilient maximization risks post-attack fragmentation, while standard connectivity assumes a failure-free environment. \textbf{(Right)} Our framework resolves this by achieving simultaneous adversarial resilience and $(\beta+1)$-vertex-connectivity, ensuring the surviving team remains structurally connected and performant.}
  \label{fig:goal}
\end{figure}

Classical submodular maximization, however, typically assumes that the selected robots remain available throughout execution. In practice, robots may be lost due to hardware failures, communication interference, denial-of-service attacks, or other adverse events. This has motivated resilient submodular maximization, where the team is selected so that the objective remains large even after a bounded number of robots are removed~\cite{zhou2019, tzoumas2017, orlin2018}. Existing resilient formulations primarily focus on preserving 
submodular objective performance
after such removals. Yet, critically, robot loss can 
disrupt the communication network among surviving robots. A team might therefore retain high sensing or coverage capacity after failures while becoming disconnected and rendered unable to continue coordinated operation.

Communication connectivity and resilience have both been studied extensively, but largely as separate requirements. In this work, we consider them jointly and ask how to select robot actions so that the team achieves a high submodular objective while remaining connected even after adversarial robot removals. We refer to this problem as \emph{Resilient \(k\)-Connected Submodular Maximization} (see Fig.~\ref{fig:goal}). The joint treatment is substantially more challenging than simply combining two existing constraints. Resilience requires accounting for the worst-case subset of removed robots, while connectivity is a global property whose feasibility depends on which robots survive.
This interaction also breaks key structures used in existing approximation analysis. In the resilient maximization framework of Tzoumas et al.~\cite{tzoumas2017}, robustness is achieved by constructing a sacrificial ``bait'' set of \(\beta\) high-value robots together with a greedily selected remainder. Under a connectivity requirement, however, the adversary determines which robots are removed, and the worst-case removal need not coincide with the bait set. Consequently, enforcing connectivity only on the non-bait robots does not guarantee that the surviving team remains connected. Moreover, the constraint of selecting one trajectory per robot together with a global vertex-connectivity requirement does not form a matroid. As a result, classical curvature-based greedy guarantees cannot be applied directly and must be re-established for the connectivity-constrained resilient setting.

We make the following contributions:
\begin{itemize}
\item We formulate resilient submodular maximization under a $(\beta{+}1)$-vertex-connectivity constraint and show that the family of partial, connectivity-feasible selections is not a matroid (Proposition~\ref{prop:nonmatroid}), identifying precisely why the classical greedy analysis fails to transfer.
\item We introduce a connectivity filter with a rigid-shift fallback across both bait and greedy stages, show it guarantees a feasible, $(\beta{+}1)$-connected output that survives any removal of size at most $\beta$ regardless of composition (Proposition~\ref{prop:connectivity}), and re-derive the curvature-greedy bound against a connectivity-respecting benchmark (Lemma~\ref{lem:greedy}), isolating the filter's entire cost into a single interpretable constant $\rho\in[0,1]$.
\item We prove a resilient, connectivity-preserving approximation guarantee (Theorem~\ref{thm:main}) and empirically corroborate the algorithm's feasibility, the price of connectivity, and the squared resilience factor (Section~\ref{sec:experiments}).
\end{itemize}


\section{Related Work}
\label{sec:rel}

Our problem sits at the intersection of three literatures, summarized in Table~\ref{tab:related}.
 
\textbf{Submodular maximization and curvature.} The classical greedy $(1-1/e)$ guarantee~\cite{nemhauser1978} was sharpened via curvature~\cite{conforti1984, iyer2013} and extended to matroid and knapsack constraints~\cite{calinescu2011, chekuri2011}. None impose a structural $k$-connectivity requirement on the selected elements.
 
\textbf{Connected submodular maximization.} Connected maximum coverage was introduced by Vandin et al.~\cite{vandin2011} and improved by Hochbaum and Rao~\cite{hochbaum2020}, with follow-up work giving a $1/2$-approximation in the geometric setting~\cite{huang2019}, an $O(\sqrt{1/k})$ bound~\cite{kuo2014}, and a two-stage heuristic for $1$-connectivity~\cite{shi2021}. None provide a provable guarantee for $k \geq 2$, with or without an adversary.
 
\textbf{Resilient submodular maximization.} Maximizing $f(A \setminus B)$ against an adversary removing $\beta$ elements was introduced and shown NP-hard by Orlin et al.~\cite{orlin2018} ($\approx 29\%$ guarantee, $\beta \leq \sqrt{2\alpha}$). Tzoumas et al.~\cite{tzoumas2017} remove this restriction for any $\beta \leq \alpha$, and Zhou et al.~\cite{zhou2019} extend it to the matroid-constrained multi-robot setting~\cite{krause2008robust}. None impose a structural requirement on the surviving graph.
 
No prior work imposes $k$-connectivity as a constraint that holds simultaneously with adversarial resilience, and this is the gap we address.

\begin{table*}[t]
\centering
\caption{Positioning relative to prior work. $k$: connectivity level, $\beta$: resilience budget, Matroid: whether the feasible family is a matroid}
\label{tab:related}
\small
\begin{tabular}{@{}lcccl@{}}
\toprule
Work & $k$ & $\beta$ & Matroid & Guarantee \\
\midrule
Nemhauser \cite{nemhauser1978}   & --      & --      & yes & $1-1/e$ \\
Conforti \cite{conforti1984}     & --      & --      & yes & $\tfrac{1}{\kf}(1-e^{-\kf})$ \\
Hochbaum--Rao \cite{hochbaum2020}& $1$     & --      & no  & radius-based \\
Kuo \cite{kuo2014}               & $\geq1$ & --      & no  & $O(\sqrt{1/k})$ \\
Shi \cite{shi2021}               & $1$     & --      & no  & heuristic \\
Orlin \cite{orlin2018}           & --      & $\leq\sqrt{2\alpha}$ & yes & $\approx0.29$ \\
Tzoumas \cite{tzoumas2017}       & --      & $\leq\alpha$ & yes & $\max(1{-}\kf,\tfrac{1}{\beta+1})\tfrac{1}{\kf}(1{-}e^{-\kf})$ \\
Zhou \cite{zhou2019}             & --      & $\leq\alpha$ & yes & $[\cdot]^2$ (matroid) \\
\textbf{This work}               & $\beta{+}1$ & $\leq\alpha$ & \textbf{no}  & $\left[\max\left\{1-\nu_f(\mathcal{I}),\,h(N,\beta)\right\}\right]^2
  \rho\,\frac{1}{\kappa_f}\left(1-e^{-\kappa_f}\right)$
  (Thm.~\ref{thm:main}) \\
\bottomrule
\end{tabular}
\end{table*}
\section{Problem Formulation}
\label{sec:formulation}
\subsection{Trajectory Model and Submodular Objective}
Let $\R=\{r_1,\dots,r_N\}$ be a team of $N$ robots. Each robot $r$ has a finite set of candidate trajectories, which we tag with the robot's identity, $\Traj_r=\{(r,\tau):\tau\in\mathcal{T}_r\}$, with $|\Traj_r|=D$. Because the robot identity $r$ is part of each element, the sets $\Traj_r$ are pairwise disjoint by construction, and we write $\Traj_\R=\bigcup_{r\in\R}\Traj_r$ for their (disjoint) union. A selection assigns each robot exactly one trajectory. We identify it with the set $S\subseteq\Traj_\R$ of chosen (robot, trajectory) pairs.

The team objective is a normalized, monotone nondecreasing, submodular set function $f:2^{\Traj_\R}\rightarrow\mathbb{R}_{\geq0}$, with $f(\emptyset)=0$.  Thus, adding a trajectory never decreases the objective, and its marginal benefit exhibits diminishing returns.  We write $\marg{s}{X}\eqdef f(X\cup\{s\})-f(X)$ for the marginal gain of trajectory $s$ relative to $X$.  One example of such an objective is the number of distinct targets covered by the selected trajectories, as discussed 
in Section~\ref{sec:experiments}.

\subsection{Communication and Failure Model}

Let $\Pcfg(s)$ denote the endpoint of trajectory $s$, and for a selection $S$ let $\Pcfg(S)=\{\Pcfg(s):s\in S\}$ be the joint configuration. Two robots communicate when their endpoints lie within range $\rc$. We define the proximity graph $G(\Pcfg(S))=(\R,E)$, where $(i,j)\in E \iff \|\Pcfg(s_i)-\Pcfg(s_j)\|\leq\rc$. Recall that the \emph{vertex connectivity} $\kappa(G)$ of a graph $G$ is the size of the smallest set of vertices whose removal disconnects $G$. By Menger's theorem~\cite{menger1927}, $\kappa(G)$ is the largest $k$ such that every pair of vertices is joined by $k$ internally vertex-disjoint paths. In our setting, each vertex is a robot, so $\kappa(G)$ is exactly the minimum number of robot failures that can disconnect the network.

We assume that in any deployment, at most $\beta$ robots may fail or be attacked~\cite{tzoumas2017, zhou2019}.  As in prior work, the identities of the failed robots are unknown at the time of trajectory selection, 
and a failure removes the affected robot's trajectory contribution.  It also removes the corresponding vertex of our communication graph $G$. 
Because every robot acts, the number of selected trajectories $\alpha := |S| = N$, and the number of greedily selected ``resilient'' robots is at least $\alpha-\beta=N-\beta$.

\subsection{Assumptions}
\begin{assumption}[Initial connectivity]
\label{ass:init}
At the start of each planning cycle, the current configuration induces a $(\beta{+}1)$-vertex-connected graph representing a multi-robot team with a prescribed $k$-vertex-connected topology as in Shi et al.~\cite{shi2025fast}.
\end{assumption}
\begin{assumption}[Endpoint connectivity]
\label{ass:endpoint}
Connectivity is required only of the committed endpoint configuration $\Pcfg(S)$. In other words, robots may temporarily disconnect mid-trajectory.
\end{assumption}
\begin{assumption}[Non-degeneracy]
\label{ass:nondegen}
Every singleton trajectory satisfies $f(s)\neq0$, so that $\nu_f(\Ifam)$ and $\kf$ are well defined and the resilient optimum is nonzero, i.e., $\foptI>0$. Section~\ref{sec:experiments} shows the excluded degenerate instances (post-attack objective collapsing to zero) are a negligible fraction of scenarios.
\end{assumption}

\subsection{Formal Problem Statement}
We synthesize the individual components of robot assignment, communication connectivity, and adversarial resilience into a single robust optimization problem. 

\begin{problem}[Resilient Connectivity-Constrained Maximization]
\label{prob:main}
Given a team of robots $\R$, candidate trajectories $\Traj_r$ for each $r \in \R$, communication radius $\rc$, and an attack budget $\beta$, find a joint trajectory selection $S^\star \subseteq \Traj_\R$ that solves:
\begin{align}
    \max_{S \subseteq \Traj_\R} \quad & \min_{\substack{B \subseteq S \\ |B| \leq \beta}} f(S \setminus B) \label{eq:prob_objective} \\
    \text{s.t.} \quad & |S \cap \Traj_r| = 1, \quad \forall r \in \R \label{eq:prob_partition} \\
    & \kappa\big(G(\Pcfg(S))\big) \geq \beta + 1 \label{eq:prob_connectivity}
\end{align}
\end{problem}
We assign each robot one trajectory so that the pre-attack graph is $(\beta{+}1)$-vertex-connected, ensuring the surviving graph stays connected after at most $\beta$ removals. The goals pull in opposite directions: coverage rewards spreading robots out to reduce overlap, while resilient connectivity rewards keeping them close with redundant links. This tension, against a worst-case adversary, breaks the matroid structure the classical greedy analysis relies on, as we show next.

\subsection{Resilient Objective and Theoretical Baseline}
For a selection $S$, a worst-case removal is $\Bstar(S)\in\arg\min_{B\subseteq S,|B|\leq\beta}f(S\setdiff B)$. We compare algorithm performance against the best connectivity-respecting resilient assignment,
\begin{equation}
  \foptI=\max_{S\in\Ifam}\min_{B\subseteq S,|B|\leq\beta}
  f(S\setdiff B),
  \label{eq:benchmark}
\end{equation}
while the unconstrained optimum is $\fopt=\max_{|S\cap\Traj_r|=1}\min_{|B|\leq\beta}f(S\setdiff B)$. Since \eqref{eq:benchmark} ranges over a subset of the selections feasible for $\fopt$, we have $\foptI\leq\fopt$. Comparing against the constrained optimal solution $\foptI$ is theoretically necessary, as standard greedy approximation proofs require the elements of the optimal set to be feasible additions to the algorithm's intermediate solutions.

\subsection{The Feasible Family and the Failure of Matroid Structure}
\label{sec:nonmatroid}
The central object is the intersection of the partition matroid (one trajectory per robot) with the connectivity constraint:
\begin{equation}
\begin{aligned}
  \Ifam = \Big\{\, S \subseteq \Traj_\R :\;& |S \cap \Traj_r| = 1 \quad \forall r \in \R, \\
  & \text{and } \kappa\big(G(\Pcfg(S))\big) \geq \beta + 1 \Big\}.
\end{aligned}
\label{eq:feasible-family}
\end{equation}
The one-trajectory-per-robot condition is a partition matroid. Intersecting it with a $k$-connectivity requirement, however, destroys the matroid structure. To make this precise, extend $\Ifam$ to a family of \emph{independent sets} $\Iind$: $S\in\Iind$ iff $S$ selects at most one trajectory per robot and the induced endpoint graph $G(\Pcfg(S))$ is $(\beta{+}1)$-vertex-connected. This is the object relevant to Algorithm~\ref{alg:rescon}'s correctness, since the stages of the proposed algorithm (Algorithm 1, introduced in Sec.~\ref{sec:method}) 
build the final trajectory selection $\Ares$ one trajectory at a time, and it is exactly the family of \emph{partial} selections encountered along this process that the classical curvature-greedy guarantee requires to form a matroid. We show it does not.
 
\begin{proposition}[$\Iind$ is not a matroid]
\label{prop:nonmatroid}
There exist a communication radius $\rc$ and a robot set such that $\Iind$ violates the hereditary property, and is therefore not the independent-set family of any matroid. Algorithm~1 constructs solutions through partial selections in $\mathcal{I}_{\mathrm{ind}}$. Its failure to form a matroid is therefore what prevents direct use of the classical curvature-greedy analysis.
\end{proposition}
 
\begin{proof}
Take $\beta=1$, so that feasibility requires $2$-vertex connectivity, and set $\rc=1$. Consider four robots $\{r_1,r_2,r_3,r_4\}$ and let $X=\{(r_1,\tau_1),(r_2,\tau_2),(r_3,\tau_3),(r_4,\tau_4)\}$ be the selection whose endpoints $\Pcfg(r_i,\tau_i)$ form a $2$-connected unit square, so $X\in\Iind$. Let $Y=X\setdiff\{(r_4,\tau_4)\}=\{(r_1,\tau_1),(r_2,\tau_2),(r_3,\tau_3)\}$. Removing one vertex from a $4$-cycle leaves a path on three vertices. Thus, $Y\notin\Iind$, violating the axiom that every subset of an independent set must itself be independent, so $\Iind$ fails to be a matroid.
\end{proof}

\section{Methodology and Theoretical Analysis}
\label{sec:method}

Our algorithm mirrors the bait-plus-greedy structure of Tzoumas et al.~\cite{tzoumas2017} and Zhou et al.~\cite{zhou2019}, augmented with a connectivity filter and a rigid-shift fallback in \emph{both} stages. The analysis is in two parts. First, connectivity and survivability (Section~\ref{sec:survivability}) follow essentially by construction and are dealt with in Proposition~\ref{prop:connectivity}. Second, the substance of the analysis is the performance bound (Sections~\ref{sec:greedy}--\ref{sec:main}), which quantifies the effect of the connectivity constraint on the usual resilient greedy choices, and how this constraint impacts the approximation guarantee.

\subsection{The Algorithm}
Algorithm~\ref{alg:rescon} (RES-CON) states the procedure precisely.
 
\begin{algorithm}[t]
\caption{RES-CON}
\label{alg:rescon}
\begin{algorithmic}[1]
\Require Robots $\R$, trajectory sets $\{T_r\}_{r\in\R}$, budget $\beta$,
         initial $(\beta{+}1)$-connected configuration $P_0$ (Assumption~\ref{ass:init})
\Ensure Feasible, $(\beta{+}1)$-connected selection $\Ares\in\Ifam$
\Statex \textit{Stage S1 --- bait}
\State $\Aone\gets\emptyset$; \ $P\gets P_0$ \Comment{shadow configuration}
\State Sort $\R$ by non-increasing single-trajectory coverage
\While{$|\Aone|<\beta$}
    \State $r\gets$ next robot in sorted order
    \State $s\gets$ highest-coverage trajectory of $r$
    \Statex \hspace{2.3em} s.t.\ $G\big(P[r{:}=\Pcfg(s)]\big)$ is $(\beta{+}1)$-connected
    \If{no such $s$ exists}
        \State $s\gets \tau_0(r)$ \Comment{stay at $P_0(r)$; Prop.~\ref{prop:connectivity}}
    \EndIf
    \State $\Aone\gets\Aone\cup\{s\}$; \ $P[r]\gets\Pcfg(s)$ \Comment{update shadow}
\EndWhile
\Statex \textit{Stage S2 --- resilient greedy}
\State $\Atwo\gets\emptyset$
\For{$k=1$ \textbf{to} $\alpha-\beta$}
    \State $s^{\star}\gets\displaystyle\arg\max_{s}\ \marg{s}{\Aone\cup\Atwo}$
    \Statex \hspace{2.3em} subject to $G\big(P[r_s{:}=\Pcfg(s)]\big)$ $(\beta{+}1)$-connected
    \If{no feasible $s^{\star}$ exists}
        \State $s^{\star}\gets \tau_0(r_s)$ \Comment{stay at $P_0(r_s)$}
    \EndIf
    \State $\Atwo\gets\Atwo\cup\{s^{\star}\}$; \ $P[r_{s^\star}]\gets\Pcfg(s^\star)$
\EndFor
\State \Return $\Ares\gets\Aone\cup\Atwo$
\end{algorithmic}
\end{algorithm}
 
\paragraph{Stage S1 (bait $\Aone$)} The robots are sorted in non-increasing order of their maximum single-trajectory coverage. The top $\beta$ robots are then processed, each assigned its highest-coverage trajectory subject to keeping the running shadow configuration $(\beta{+}1)$-connected (line 5), with the stay-in-place fallback (i.e., returning the robot to its Assumption~\ref{ass:init} initial endpoint $P_0(r)$) invoked when no such trajectory exists (line 7; see Proposition~\ref{prop:connectivity}). Ordering by standalone rather than marginal value may assign overlapping bait trajectories, but this affects neither feasibility nor Theorem~\ref{thm:main} as the redundancy is priced into the guarantee through $\nuf\to1$. Because the fallback leaves the shadow configuration unchanged at that robot's coordinate, it trivially preserves whatever connectivity level the configuration held prior to the attempted commit.

\paragraph{Stage S2 (resilient set $\Atwo$)} Each of the remaining $\alpha-\beta$ robots greedily selects the connectivity-feasible trajectory of maximal marginal gain (line 13), executing the rigid fallback 
when no feasible candidate exists (line 15).
 
\paragraph{Output} $\Ares=\Aone\cup\Atwo$ (line 19).

\paragraph{Complexity} Each commit evaluates at most $D$ candidate trajectories, and for each it tests whether $\kappa(G)\geq\beta{+}1$ on an $N$-vertex graph. This test runs in $O\big((\beta{+}1)\,|E|\,\sqrt{N}\big)$ time, with $|E|=O(N^2)$ in the worst case. Over the $N$ commits, RES-CON therefore runs in low-order polynomial time, $O\big(N\,D\,(\beta{+}1)\,|E|\,\sqrt{N}\big)$, matching the connectivity-blind resilient greedy up to the per-step connectivity factor.
 
\subsection{Connectivity and Survivability}
\label{sec:survivability}
Every commitment in Algorithm~\ref{alg:rescon} preserves $(\beta{+}1)$-connectivity (or falls back to a move that trivially does), so output connectivity and survivability under attack follow directly from the construction.
 
\begin{proposition}[Connectivity and survivability]
\label{prop:connectivity}
Algorithm~\ref{alg:rescon} always outputs some $\Ares\in\Ifam$. Moreover, for every removal $B\subseteq\R$ with $|B|\leq\beta$ (whether $B$ consists of bait robots, greedy robots, or any mixture of the two), the induced subgraph 
$G(\Pcfg(\Ares)) \setminus B$
is connected.
\end{proposition}
\begin{proof}
\emph{Feasibility.} Let $P^{(0)}=P_0$ denote the shadow configuration prior to any commits, and let $P^{(i)}$ denote the shadow configuration after the $i$-th commit (previously committed robots at their chosen endpoints, not-yet-committed robots still at their $P_0$ position). We show by induction on $i$ that $P^{(i)}$ is $(\beta{+}1)$-connected for every $i$. The base case $P^{(0)}=P_0$ holds by Assumption~\ref{ass:init}. For the inductive step, suppose $P^{(i-1)}$ is $(\beta{+}1)$-connected and consider the $i$-th commit, for robot $r$. If a trajectory $s$ exists such that $G\big(P^{(i-1)}[r{:=}\Pcfg(s)]\big)$ is $(\beta{+}1)$-connected, the algorithm selects it directly, and $P^{(i)}=P^{(i-1)}[r{:=}\Pcfg(s)]$ is $(\beta{+}1)$-connected by construction. Otherwise, the algorithm falls back to $\tau_0(r)$, i.e., $r$ remains at $P_0(r)$, so $P^{(i)}=P^{(i-1)}$ is unchanged and remains $(\beta{+}1)$-connected by the inductive hypothesis. Every commit in both Stage S1 and Stage S2 falls into one of these two cases, so by induction over the full commit sequence the final shadow configuration $P^{(N)}=\Pcfg(\Ares)$ is $(\beta{+}1)$-connected, so $\Ares\in\Ifam$.
 
\emph{Survivability.} A graph with $\kappa(G)\geq k$ remains connected after deleting any set of fewer than $k$ vertices, equivalently (by Menger's theorem) every pair of vertices is joined by $k$ internally vertex-disjoint paths~\cite{menger1927}. Taking $k=\beta+1$ and $|B|\leq\beta$ yields connectivity of $G(\Pcfg(\Ares)) \setminus B$ for \emph{any} $B$ of size at most $\beta$, regardless of its composition.
\end{proof}
 
Proposition~\ref{prop:connectivity} collapses a requirement over all $\binom{N}{\beta}$ possible removals, $\forall B$ with $|B|\leq\beta$, $G(\Pcfg(\Ares)) \setminus B$ connected, into a single scalar test, $\kappa(G)\geq\beta+1$, checked by the connectivity filter once per commit rather than enumerated per attack.

\subsection{Curvature}
We use the \emph{constrained curvature} of Zhou et al.~\cite{zhou2019} to quantify the worst-case informational overlap between selected trajectories. For a feasible family $\Ifam$ and a monotone submodular $f$ with $f(s)\neq0$ for every singleton trajectory,
\begin{equation}
  \nu_f(\Ifam)=1-\min_{S\in\Ifam}\min_{s\in S}
  \frac{f(S)-f(S\setdiff\{s\})}{f(s)}\in[0,1].
  \label{eq:constrained-curvature}
\end{equation}

Here $\nu_f = 0$ means the trajectories provide independent information, while $\nu_f \to 1$ indicates severe diminishing returns. We also define the global curvature, $\kf=1-\min_{v}\frac{f(\Traj_\R)-f(\Traj_\R\setdiff\{v\})}{f(v)}$~\cite{conforti1984, iyer2013}, which similarly bounds overlap across the entire ground set, for the greedy recursion.

\subsection{A Greedy Bound under the Connectivity Constraint}
\label{sec:greedy}
Because $\Ifam$ is not a matroid, the Conforti--Cornu\'ejols curvature-greedy guarantee does not immediately transfer. We re-derive it with respect to the constrained optimum, isolating the cost of the connectivity filter into a single interpretable constant.
 
\begin{definition}[Price of connectivity per step]
\label{def:rho}
At greedy step $k$, let $\gunc_k$ be the largest marginal gain over all remaining single-robot trajectories, and let $\gcon_k$ be the largest marginal gain over those trajectories that keep the committed graph $(\beta{+}1)$-connected. Define:
\begin{equation}
  \rho=\min_{k:\,\gunc_k>0}\frac{\gcon_k}{\gunc_k}\in[0,1].
  \label{eq:rho}
\end{equation}
\end{definition}
A feasible candidate always exists at every step (Proposition~\ref{prop:connectivity}), so $\gcon_k$ is well defined; $\rho$ is the worst-case fraction of greedy marginal value retained under the constraint, with $\rho=1$ when connectivity never throttles the greedy choice and $\rho\to0$ when it does so severely.
 
Let $\AoptI\in\arg\max_{A:\, \Aone\cup A\in\Ifam,|A|\leq\alpha-\beta}f(A)$ be the best feasible completion of the bait, where recall that $\Aone$ and $\Atwo$ denote the S1 and S2 outputs, respectively.

\begin{lemma}[Constrained curvature-greedy bound]
\label{lem:greedy}
Under Assumptions~\ref{ass:init}--\ref{ass:nondegen},
\begin{equation}
  f(\Atwo)\;\geq\;\rho\cdot\frac{1}{\kf}\big(1-e^{-\kf}\big)\,
  f(\AoptI).
  \label{eq:greedy-bound}
\end{equation}
\end{lemma}
\begin{proof}
Index the greedy steps $k=1,\dots,\alpha-\beta$, and write $\Atwo^{(k)}$ for the partial greedy set after $k$ commits. Let $O\eqdef\AoptI$, with $|O|\leq\alpha-\beta$.

\emph{Step (i): submodular decomposition relative to the feasible optimum.} By monotonicity and submodularity, at each step $k$,
\begin{equation*}
\begin{aligned}
  f(O)-f(\Atwo^{(k)}) &\leq f(O\cup\Atwo^{(k)})-f(\Atwo^{(k)}) \\
  &\leq \sum_{s\in O\setdiff\Atwo^{(k)}} \marg{s}{\Atwo^{(k)}}.
\end{aligned}
\end{equation*}
Every $s\in O$ is, by definition of $O$ as a feasible completion of the bait, a candidate the greedy step could have considered, so it lies in the feasible pool whose best element defines $\gunc_{k+1}$. This decomposition is exactly why the bound is stated against the constrained optimum $\foptI$: elements of an unconstrained optimum could be infeasible to add and would break it.

\emph{Step (ii): curvature relates the marginal sum to a single greedy step.} Applying the curvature definition to the marginals bounded in Step (i) and telescoping via the standard Conforti--Cornu\'ejols recursion~\cite{conforti1984} yields $f(\Atwo')\geq\tfrac{1}{\kf}(1-e^{-\kf})f(O)$ for the hypothetical greedy $\Atwo'$ that at each step takes the unconstrained best feasible marginal $\gunc_k$.
Crucially, this recursion remains valid because submodularity and curvature are properties of $f$ alone; they do not require the optimal set $O$ to be drawn from a strict matroid structure.


\emph{Step (iii): discount by $\rho$.} The actual algorithm commits the connectivity-feasible best, achieving marginal $\gcon_k\geq\rho\gunc_k$ at every step by \eqref{eq:rho}. Replacing each step's progress by its $\rho$-fraction and re-telescoping yields $f(\Atwo) \geq \rho f(\Atwo')$. This scales the bound in (ii) uniformly by $\rho$, giving \eqref{eq:greedy-bound}.
\end{proof}

\subsection{Main Theorem}
\label{sec:main}

To bound Algorithm~\ref{alg:rescon}'s performance, we first use Lemma 1 to bound $f(\Atwo)$ in terms of its constrained optimum $f(\AoptI)$ and then apply Lemma 2 (proven below) to bound $f(\AoptI)$ in terms of the true worst-case optimum $\foptI$.

\begin{theorem}[Resilient, connectivity-preserving approximation]
\label{thm:main}
Under Assumptions~\ref{ass:init}--\ref{ass:nondegen}, the output $\Ares$ of Algorithm~\ref{alg:rescon} satisfies,
\begin{flalign}
  f\big(\Ares\setdiff\Bstar(\Ares)\big) &\geq \big[\max(1-\nuf,\hNb)\big]^2\cdot\rho \nonumber \\
  &\quad \cdot \frac{1}{\kf}\big(1-e^{-\kf}\big)\,\foptI, \label{eq:main-bound}
\end{flalign}
where $\hNb=\max\!\big(\tfrac{1}{1+\beta},\tfrac{1}{N-\beta}\big)$.
\end{theorem}
\begin{proof}
We follow the bait/greedy case split over $\Ifam$. Write $\Bstar=\Bstar(\Ares)$,
$\Bstar_1=\Bstar\cap\Aone$, and $\Bstar_2=\Bstar\cap\Atwo$.
 
\emph{Case A: $\Bstar_2=\emptyset$.} Then $\Ares\setdiff\Bstar=\Atwo$, so by Lemma~\ref{lem:greedy} and Lemma~\ref{lem:second}, $f(\Ares\setdiff\Bstar)=f(\Atwo)\geq\rho\tfrac{1}{\kf}(1-e^{-\kf}) \max(1-\nu_f(\Ifam), h(N, \beta))\foptI$. Since $\max(1-\nuf,\hNb)\leq1$, the squared bound in \eqref{eq:main-bound} holds.
 
\emph{Case B: $\Bstar_2\neq\emptyset$.} Because $|\Bstar|=|\Aone|=\beta$ and at least one greedy element is removed, at least one bait element survives. The bait robots are those of largest single-trajectory coverage, each assigned its highest-coverage connectivity-feasible trajectory. Since the greedy elements are drawn from the same connectivity-feasible universe, each surviving bait element dominates each removed greedy element in absolute single-trajectory coverage. Applying the constrained-curvature counterpart of Tzoumas et al.'s Lemma~1 (using $\nuf$ in place of the global curvature) and value-domination gives the modularity-side bound $f(\Ares\setdiff\Bstar)\geq(1-\nuf)f(\Atwo)$. Independently, the resilience accounting of Zhou et al.~\cite{zhou2019} yields the cardinality-side bound $f(\Ares\setdiff\Bstar)\geq\hNb f(\Atwo)$. Taking the larger of the two gives $f(\Ares \setminus \Bstar) \ge \max(1-\nu_f(\Ifam), h(N, \beta))f(A_2)$. 
Substituting $f(\Atwo)$ using Lemma 1, and then substituting $f(\AoptI)$ using Lemma 2, yields the desired squared bound in~\eqref{eq:main-bound}.
Because the $\max$ factor is $\le 1$, this squared bound is the worst-case minimum, ensuring  \eqref{eq:main-bound} holds in all cases.
\end{proof}
 
We now formally bound the constrained optimum $f(\AoptI)$ against the worst-case optimum $\foptI$.

\begin{lemma}[Second resilience factor]
\label{lem:second}
$f(\AoptI)\geq\max(1-\nuf,\hNb)\cdot\foptI.$
\end{lemma}
\begin{proof}
Let $S^{\star}\in\Ifam$ attain $\foptI$, with worst-case removal
$\Bstar(S^{\star})$, and split it into the $\beta$ highest-singleton-value elements (the \emph{virtual bait} $\tilde A_1$) and the remainder $\tilde A_2\eqdef S^{\star}\setdiff\tilde A_1$. Since $\tilde A_2$ is a feasible completion of a size-$\beta$ bait, it competes in the maximization defining $\AoptI$, so $f(\AoptI)\geq f(\tilde A_2)$. Removing the $\beta$ highest-value elements extracts the maximum possible value under the curvature bound defined by Zhou et al.~\cite{zhou2019}, guaranteeing that the remainder retains at least 
$f(\tilde A_2)\geq (1-\nuf)\foptI$. The survivor-count accounting of Zhou et al.~\cite{zhou2019} gives the cardinality bound $f(\tilde A_2)\geq\hNb \foptI$. Taking the larger yields the claim.
\end{proof}

We formally justify multiplying the bounds from Lemmas~\ref{lem:greedy} and~\ref{lem:second} by ensuring the connectivity penalty $\rho$ is conditionally independent of the Stage S1 trajectory selections:

\begin{proposition}[Stage independence under intra-stage filtering]
\label{prop:independence}
Suppose the connectivity filter is \emph{intra-stage}: the feasible S2 completions of the bait depend on S1 only through the fixed endpoint set $\Pcfg(\Aone)$, and the retention ratios $\gcon_k/\gunc_k$ defining $\rho$ are evaluated against that same configuration. Then the loss from the S1 bait commitment (Lemma~\ref{lem:second}) and the S2 loss (Lemma~\ref{lem:greedy}) compose multiplicatively, yielding the squared factor in \eqref{eq:main-bound}.
\end{proposition}
\begin{proof}
The bounds from Lemma~\ref{lem:greedy} and Lemma~\ref{lem:second} chain through the common term $f(\AoptI)$. The intra-stage hypothesis confines the influence of the S1 choice to the endpoints $\Pcfg(\Aone)$; freezing these renders $\rho$ conditionally independent of the specific S1 trajectories, so the two worst-case fractions multiply. Combining with the Case~B inequality $f(\Ares\setdiff\Bstar)\geq \max(1-\nuf,\hNb)f(\Atwo)$ yields the squared bound in Theorem~\ref{thm:main}.
\end{proof}
 
\begin{remark}[Scope of the square]
\label{rem:square}
The squared factor holds under Assumptions~\ref{ass:init}--\ref{ass:nondegen} and the intra-stage hypothesis of Proposition~\ref{prop:independence} on every non-degenerate instance. When $\foptI = 0$, \eqref{eq:main-bound} is vacuously true but the ratio is the indeterminate $0/0$; Assumption~\ref{ass:nondegen} excludes exactly this floor-collapse set. Section~\ref{sec:experiments} confirms that the $\approx 0.7\%$ of swept instances violating the squared bound are precisely these degenerate cases, independent of the connectivity regime.
\end{remark}
 
\begin{corollary}[Bridge to the unconstrained optimum]
\label{cor:bridge}
Suppose there is $\gamma\in(0,1]$ with $\foptI\geq\gamma\fopt$ ($\gamma=1$ whenever some unconstrained resilient optimum is already $(\beta{+}1)$-connected). Then $f(\Ares\setdiff\Bstar(\Ares))\geq [\max(1-\nuf,\hNb)]^2\rho\tfrac{1}{\kf}(1-e^{-\kf})\gamma\fopt$.
\end{corollary}

\section{Experiments}
\label{sec:experiments}

Our experiment objectives are (i) to characterize the connectivity penalty $\rho$ across communication densities and team sizes, establishing that it remains bounded away from zero in realistic instances, (ii) to validate the feasibility of Algorithm~\ref{alg:rescon}'s outputs (Proposition~\ref{prop:connectivity}), (iii) to test the squared resilience factor derived in \eqref{eq:main-bound}, and (iv) to evaluate the algorithm's operational resilience and connectivity maintenance over multi-planning cycle deployments.

\subsection{System Model}
We instantiate a multi-robot active target tracking formulation, following Shi et al.~\cite{shi2021}. A team of $N$ aerial robots flies at fixed altitude in a bounded arena $W=[0,L]^2$. Robot $i$ occupies position $x_i\in W$ with a downward camera of sensing radius $\rs$ under a condition where two robots communicate if and only if $\|x_i-x_j\|\leq\rc$. During each planning cycle, robot $i$ may move within a disc of radius $\rmove$, discretized into $D$ candidate trajectories, while $M$ ground targets move by a single-integrator model with reflecting boundaries. The objective is the number of distinct targets observed by at least one robot,
\begin{equation}
  f(A)=\Big|\bigcup_{i\in A}T(x_i)\Big|,\quad
  T(x_i)=\{j:\|p^j-x_i\|\leq\rs\},
  \label{eq:cover}
\end{equation}
a monotone submodular coverage function with total curvature $\kf=1-1/M$. This closed form gives direct control over curvature. After the team commits $A$, a rational $\beta$-adversary removes $\Bstar(A)=\arg\min_{|B|\leq\beta}f(A\setdiff B)$, which may be computed exactly for small teams. 

\begin{figure*}[t]
  \centering
  \includegraphics[width=0.67\textwidth]{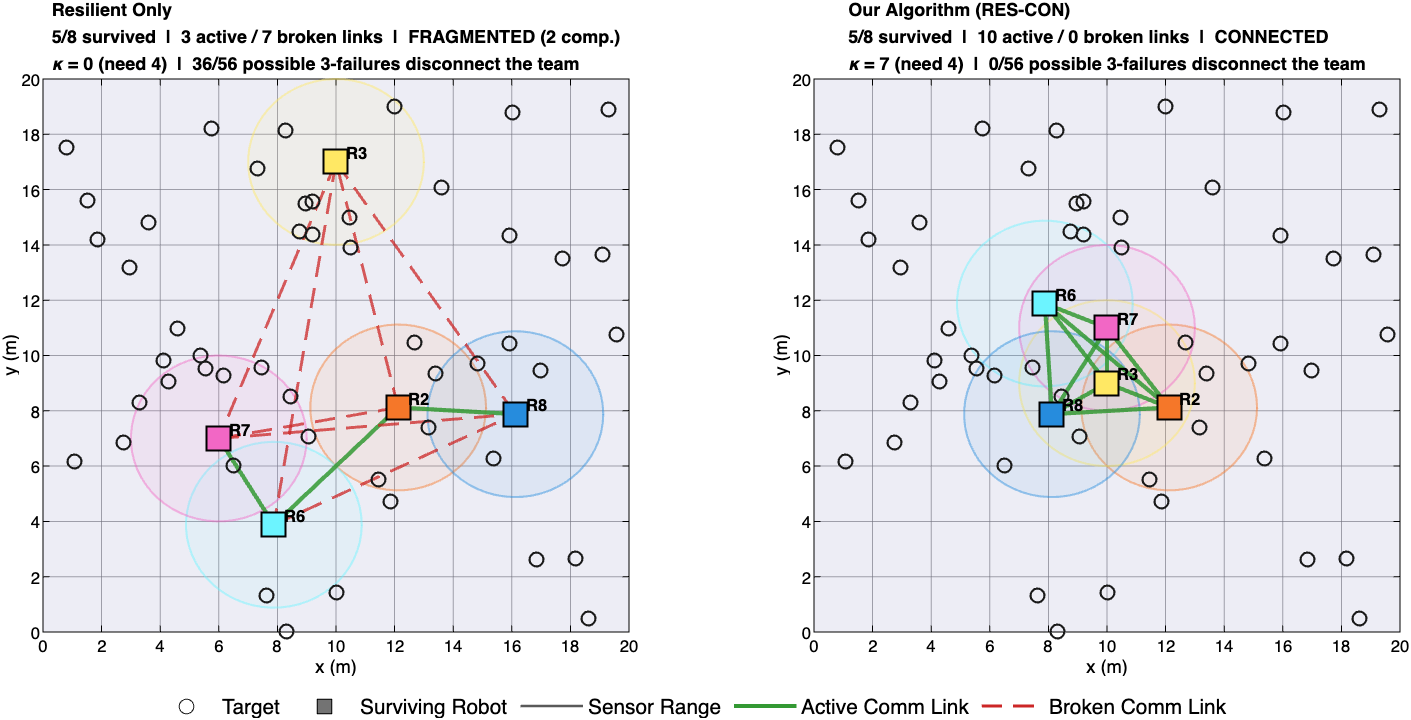}
  \caption{Post-attack topology after removing 3 bait robots from an 8-robot team. \textbf{(a) Resilient Only} \cite{zhou2019}: unconstrained $S_2$ (resilient set) selection leaves survivors fragmented ($\kappa=0$, 2 components); 36/56 possible 3-robot failures would disconnect them.
  \textbf{(b) RES-CON}: adding a $k$-connectivity constraint ($k=4$) keeps survivors fully connected ($\kappa=7$), with 0/56 possible failures.}
  \label{fig:comp}
\end{figure*}

We compare against baselines spanning the connectivity/resilience design space: \textbf{Pure Greedy} (no constraints), \textbf{Resilient-Only} (Tzoumas et al.~\cite{tzoumas2017}, no connectivity), \textbf{$1$-Connected-Only} (Hochbaum--Rao~\cite{hochbaum2020}, CSM from Shi et al.~\cite{shi2021}), \textbf{$(\beta{+}1)$-Connected-Only} (connectivity, no adversary), and \textbf{Fallback-ablated} (ours without the rigid-shift fallback). We denote our method \textbf{RES-CON}. Figure~\ref{fig:comp} visualizes a post-attack comparison between \textbf{RES-CON} and the \textbf{Resilient-Only} baseline.

\subsection{Measuring the Price of Connectivity \texorpdfstring{$\rho$}{rho}}
\label{sec:exp-rho}
At each greedy step, we record the marginal gains $\gunc_k$ and $\gcon_k$ and the running minimum of their ratio, i.e., the realized $\rho$~\eqref{eq:rho}. We sweep density $\rc/\rmove$ across sparse-to-dense regimes and use $N\in\{6,10,20,40\}$ and $\beta \in \{N/4, N/2\}$, over $100$ random initializations per parameter set.

\begin{figure*}[t]
  \centering
  \includegraphics[width=0.63\textwidth]{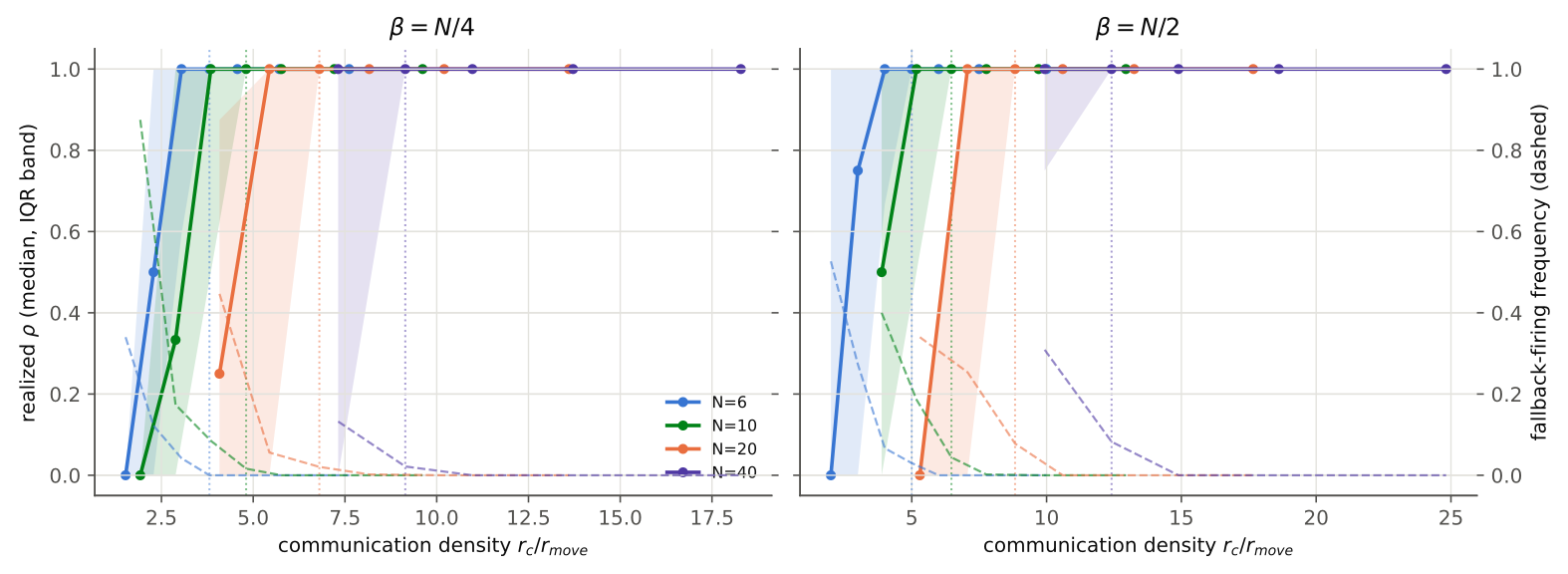}
  \caption{Realized price of connectivity $\rho$ (solid: median; shaded band: interquartile range) and fallback-firing frequency (dashed) versus communication density, for $\beta=N/4$ (left) and $\beta=N/2$ (right). The dotted verticals mark the $(\beta{+}1)$-connectivity percolation threshold. Above the threshold $\rho$ is a flat $1.0$ and the fallback frequency collapses to zero. Below it, $\rho$ degrades as the fallback activates.}
  \label{fig:rho}
\end{figure*}

Above the connectivity percolation threshold, the realized $\rho$ is a flat median of $1.0$ in every $(N,\beta)$ cell (Fig.~\ref{fig:rho}), with fallback frequency collapsing to approximately zero. Meanwhile, below the threshold, it degrades sharply (median $\to0$, with the fallback invoked up to $44\%$ of steps). This empirical result demonstrates that \eqref{eq:main-bound} is non-vacuous in the operating regime: when Assumption~\ref{ass:init} is satisfied, the connectivity filter imposes no significant penalty to the objective score (i.e., $\rho \approx 1$).

To distinguish whether $\rho\!=\!1$ reflects a filter that is inexpensive to satisfy or one that is invoked only rarely, we measure the per-step candidate-\emph{rejection} rate. Above threshold, the filter rejects at least one candidate on only $1.67\%$ of steps. Below threshold, by contrast, the filter rejects at least one selection on $32\%$ of steps. Thus, we conclude that the connectivity constraint is \emph{essentially inactive} above the percolation threshold and becomes binding precisely in the sparse regime where connectivity is genuinely at risk.

\subsection{Feasibility Validation (Proposition~\ref{prop:connectivity})}
\label{sec:exp-feas}
Across $3{,}619$ instances, the committed team is $(\beta{+}1)$-connected on $100\%$ of outputs, confirming Proposition~\ref{prop:connectivity}. For comparison, an ablation without the Stage S1 connectivity filter fails to maintain connectivity on 
$\approx29\%$ of instances below threshold. The S1 filter repairs this at negligible cost without measurably affecting the post-attack objective ($93.5\%$ of instances identical to the unfiltered variant), confirming that guaranteed connectivity does not trade away resilience.

\subsection{Testing the Squared Resilience Factor}
\label{sec:exp-square}
For small teams ($N\leq12$) where $\foptI$ is computable by brute force, we measure the empirical ratio $f(\Ares\setdiff\Bstar)/\foptI$ and compare against 
the squared bound of \eqref{eq:main-bound}. 


\begin{figure}[t]
  \centering
  \includegraphics[width=0.75\columnwidth]{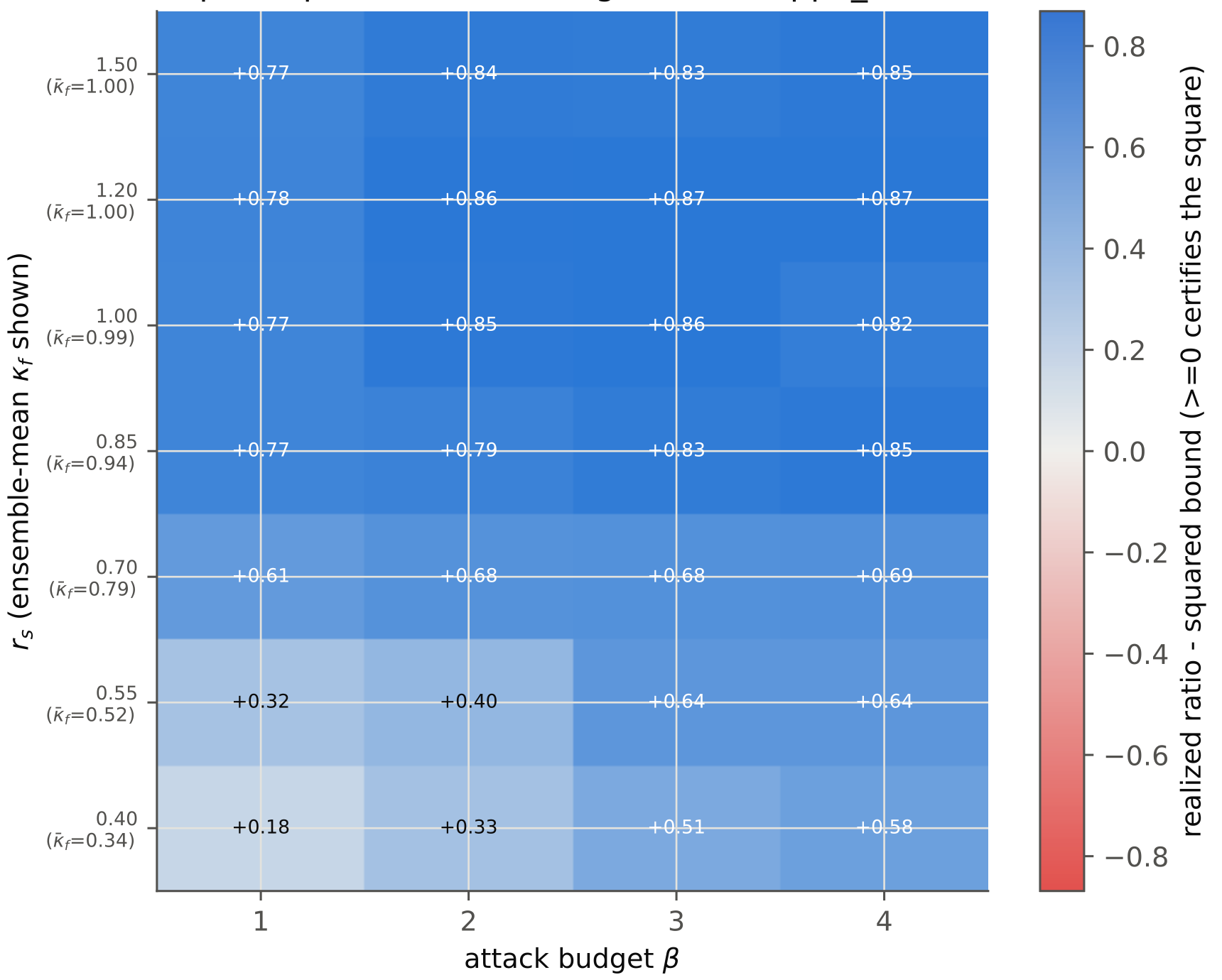}
  \caption{Margin (realized ratio $-$ squared bound) over the $(\kf,\beta)$ plane, where the non-negativity certifies the squared factor. Every cell is positive, indicating the squared bound holds across the entire swept plane on non-degenerate instances.}
  \label{fig:heatmap}
\end{figure}

The squared bound is respected on $\approx99.3\%$ of instances, with margin uniformly positive across the $(\kf,\beta)$ plane, from $+0.18$ to $+0.87$ (Fig.~\ref{fig:heatmap}). The $\approx0.7\%$ exceptions are exactly the degenerate $\foptI=0$ instances of Remark~\ref{rem:square}, independent of the connectivity regime.

\subsection{The Connectivity-Resilience Trade-off}
\label{sec:exp-tracking}
Our central claim is that RES-CON preserves connectivity under repeated adversarial attack where connectivity-blind methods fragment. We test this over a $T$-planning cycle horizon with moving targets, sweeping $\rc$ relative to the percolation threshold and measuring both the worst-case-survivor disconnection rate and cumulative target tracking.

\begin{figure}[t]
  \centering
  \includegraphics[width=0.8\columnwidth]{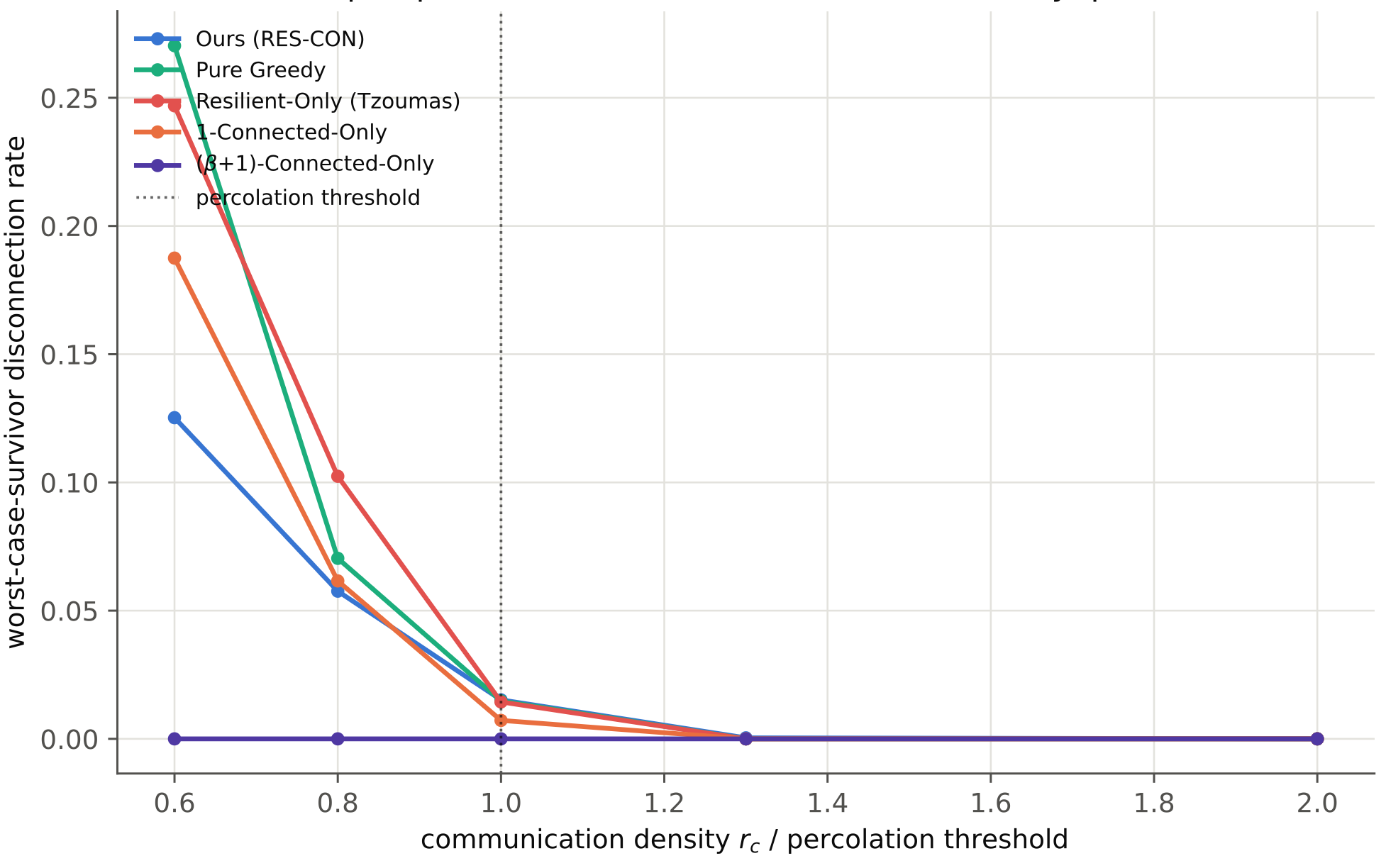}
  \caption{Worst-case-survivor disconnection rate versus communication density (normalized to the percolation threshold).}
  \label{fig:disconnection}
\end{figure}

\textbf{Connectivity:} Below the percolation threshold, RES-CON disconnects on $12.5\%$ of planning cycles at the sparsest density versus $24.7\%$ for Resilient-Only and $27\%$ for Pure Greedy (Fig.~\ref{fig:disconnection}). At and above threshold every method disconnects at near-zero rate.

\begin{figure}[t]
  \centering
  \includegraphics[width=0.8\columnwidth]{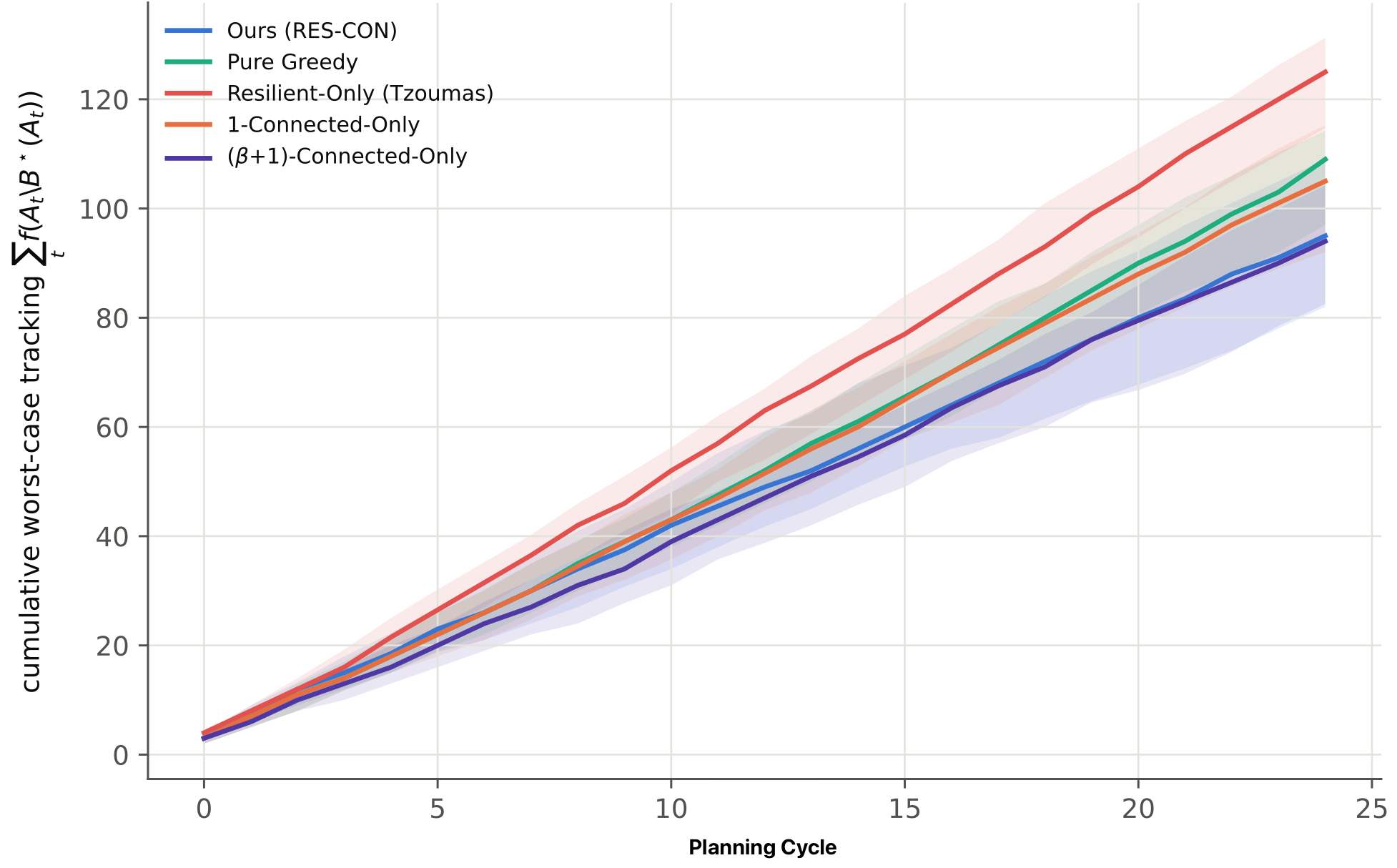}
  \caption{Cumulative worst-case tracking below threshold
    ($\rc/\rmove=5.2$). Resilient-Only accrues the most tracking value by ignoring connectivity, while RES-CON accepts a modest tracking reduction in exchange for roughly halved disconnection}
  \label{fig:tracking}
\end{figure}

\textbf{Trade-off:} Below threshold, Resilient-Only secures the highest cumulative tracking by disregarding the communication graph, but since a fragmented team cannot coordinate under our assumptions, that advantage is misleading. RES-CON instead accepts a modest tracking reduction for the roughly halved disconnection rate as shown in Fig.~\ref{fig:disconnection} and in Fig.~\ref{fig:tracking}, delivering connectivity-guaranteed tracking at a small coverage cost. At and above threshold, the tracking curves of all methods coincide, consistent with the connectivity constraint being inactive there (Section~\ref{sec:exp-rho}).


\section{Conclusion and Future Work}
\label{sec:conclusion}

We introduced resilient submodular maximization under a $(\beta{+}1)$-vertex-connectivity constraint, jointly capturing adversarial robustness and communication network integrity for multi-robot teams. Several directions remain open. First, the squared resilience factor is established under the intra-stage hypothesis of Proposition~\ref{prop:independence} and corroborated empirically. A proof that this hypothesis holds unconditionally for disk-graph geometries under bounded motion would bolster the idea. Second, deriving a geometric lower bound for $\rho$ 
would convert the penalty into a computable quantity, making \eqref{eq:main-bound} a fully unconditional guarantee on non-degenerate instances. Finally, a hardware demonstration on an aerial or ground fleet would validate the connectivity guarantee under real communication models.

\section*{Acknowledgments}

This work was sponsored in part by the U.S. Army DEVCOM Army Research Laboratory under Cooperative Agreement Number W911NF-17-2-0181.

\bibliographystyle{plain}

\begin{thebibliography}{99}

\bibitem{marques2015critical}
M.~Dunbabin and L.~Marques, ``Robots for environmental monitoring: Significant advancements and applications,'' \emph{IEEE Robotics \& Automation Magazine}, vol.~19, no.~1, pp.~24--39, 2015.

\bibitem{robin2016}
C.~Robin and S.~Lacroix, ``Multi-robot target detection and tracking: Taxonomy and survey,'' \emph{Autonomous Robots}, vol.~40, no.~4, pp.~729--760, 2016.
 
\bibitem{queralta2020}
J.~Pe\~na Queralta, J.~Taipalmaa, B.~Can Pullinen, V.~K. Sarker, T.~Nguyen Gia, H.~Tenhunen, M.~Gabbouj, J.~Raitoharju, and T.~Westerlund, ``Collaborative multi-robot search and rescue: Planning, coordination, perception, and active vision,'' \emph{IEEE Access}, vol.~8, pp.~191617--191643, 2020.
 
\bibitem{burgard2005}
W.~Burgard, M.~Moors, C.~Stachniss, and F.~E. Schneider, ``Coordinated multi-robot exploration,'' \emph{IEEE Transactions on Robotics}, vol.~21, no.~3, pp.~376--386, 2005.

\bibitem{krausegolovin2014}
A.~Krause and D.~Golovin, ``Submodular function maximization,'' in \emph{Tractability: Practical Approaches to Hard Problems}, Cambridge University Press, pp.~71--104, 2014.

\bibitem{krause2008robust}
A.~Krause, H.~B. McMahan, C.~Guestrin, and A.~Gupta, ``Robust submodular observation selection,'' \emph{Journal of Machine Learning Research}, vol.~9, no.~12, 2008.

\bibitem{shi2021}
G.~Shi, I.~E. Rabban, L.~Zhou, and P.~Tokekar, ``Communication-aware multi-robot coordination with submodular maximization,'' in \emph{Proc. IEEE International Conference on Robotics and Automation (ICRA)}, pp.~8955--8961, 2021.

\bibitem{zhou2019}
L.~Zhou, V.~Tzoumas, G.~J. Pappas, and P.~Tokekar, ``Resilient active target tracking with multiple robots,'' \emph{IEEE Robotics and Automation Letters}, vol.~4, no.~1, pp.~129--136, 2019.

\bibitem{dames2017}
P.~Dames, P.~Tokekar, and V.~Kumar, ``Detecting, localizing, and tracking an unknown number of moving targets using a team of mobile robots,'' \emph{The International Journal of Robotics Research}, vol.~36, no.~13-14, pp.~1540--1553, 2017.

\bibitem{nemhauser1978}
G.~L. Nemhauser, L.~A. Wolsey, and M.~L. Fisher, ``An analysis of approximations for maximizing submodular set functions---I,'' \emph{Mathematical Programming}, vol.~14, no.~1, pp.~265--294, 1978.

\bibitem{tzoumas2017}
V.~Tzoumas, K.~Gatsis, A.~Jadbabaie, and G.~J. Pappas, ``Resilient monotone submodular function maximization,'' in \emph{Proc. IEEE Conference on Decision and Control (CDC)}, pp.~1362--1367, 2017.

\bibitem{orlin2018}
J.~B. Orlin, A.~S. Schulz, and R.~Udwani, ``Robust monotone submodular function maximization,'' \emph{Mathematical Programming}, vol.~172, no.~1, pp.~505--537, 2018.

\bibitem{conforti1984}
M.~Conforti and G.~Cornu\'ejols, ``Submodular set functions, matroids and the greedy algorithm,'' \emph{Discrete Applied Mathematics}, vol.~7, no.~3, pp.~251--274, 1984.

\bibitem{iyer2013}
R.~K. Iyer, S.~Jegelka, and J.~A. Bilmes, ``Curvature and optimal algorithms for learning and minimizing submodular functions,'' in \emph{Advances in Neural Information Processing Systems}, vol.~26, 2013.

\bibitem{calinescu2011}
G.~Calinescu, C.~Chekuri, M.~P\'al, and J.~Vondr\'ak, ``Maximizing a monotone submodular function subject to a matroid constraint,'' \emph{SIAM Journal on Computing}, vol.~40, no.~6, pp.~1740--1766, 2011.

\bibitem{chekuri2011}
C.~Chekuri, J.~Vondr\'ak, and R.~Zenklusen, ``Submodular function maximization via the multilinear relaxation and contention resolution schemes,'' in \emph{Proc. 43rd ACM Symposium on Theory of Computing (STOC)}, pp.~783--792, 2011.

\bibitem{vandin2011}
F.~Vandin, E.~Upfal, and B.~J. Raphael, ``Algorithms for detecting significantly mutated pathways in cancer,'' \emph{Journal of Computational Biology}, vol.~18, no.~3, pp.~507--522, 2011.

\bibitem{hochbaum2020}
D.~S. Hochbaum and X.~Rao, ``Approximation algorithms for connected maximum coverage problem for the discovery of mutated driver pathways in cancer,'' \emph{Information Processing Letters}, vol.~158, art.~105940, 2020.

\bibitem{huang2019}
C.-C. Huang, M.~Mari, C.~Mathieu, J.~S.~B. Mitchell, and N.~H. Mustafa, ``Maximizing covered area in the Euclidean plane with connectivity constraint,'' in \emph{APPROX/RANDOM}, pp.~32:1--32:21, 2019.

\bibitem{kuo2014}
T.-W. Kuo, K.~C.-J. Lin, and M.-J. Tsai, ``Maximizing submodular set function with connectivity constraint: Theory and application to networks,'' \emph{IEEE/ACM Transactions on Networking}, vol.~23, no.~2, pp.~533--546, 2014.

\bibitem{shi2025fast}
G.~Shi, M.~Ishat-E-Rabban, G.~Bonner, and P.~Tokekar, ``Fast $k$-connectivity restoration in multi-robot systems for robust communication maintenance: Algorithmic and learning-based solutions,'' \emph{Autonomous Robots}, vol.~49, Art.~no.~34, 2025.

\bibitem{menger1927}
K.~Menger, ``Zur allgemeinen Kurventheorie,'' \emph{Fundamenta Mathematicae}, vol.~10, no.~1, pp.~96--115, 1927.

\end{thebibliography}

\end{document}